\documentclass[11pt,a4paper]{article}
\usepackage[margin=2.5cm,bottom=3.cm]{geometry}	

\usepackage[utf8]{inputenc} 
\usepackage[T1]{fontenc}    
\usepackage{lmodern}
\usepackage{hyperref}       
\hypersetup{colorlinks,
            linkcolor=blue,
            citecolor=blue,
            urlcolor=magenta,
            linktocpage,
            plainpages=false}
\usepackage{url}            
\usepackage{booktabs}       
\usepackage{amsfonts}       
\usepackage{nicefrac}       
\usepackage{microtype}      
\usepackage{amsmath, amssymb, amsthm, graphicx, url} 

\usepackage{natbib}
\usepackage[algo2e,ruled]{algorithm2e}
\usepackage{parskip} 
\usepackage{xcolor} 
\usepackage{enumitem} 
\usepackage{nicefrac} 

\title{An Efficient Algorithm for Cooperative Semi-Bandits}

\author{
  \textbf{Riccardo Della Vecchia}\\
  Artificial Intelligence Lab, Institute for Data Science \& Analytics\\
    Bocconi University, Milano, Italy
  \and
  \textbf{Tommaso Cesari}\\
  Toulouse School of Economics (TSE), Toulouse, France\\
    Artificial and Natural Intelligence Toulouse Institute (ANITI), Toulouse, France 
}

\usepackage{mleftright} 
    \mleftright
\usepackage{microtype} 
\usepackage{booktabs} 
\usepackage{enumitem} 
\usepackage{nicefrac} 
\usepackage{tikz} 
    \usetikzlibrary{external} 
\usepackage{algorithm,algpseudocode} 
\usepackage{comment} 
\usepackage{cancel} 

\newcommand{\cA}{\mathcal{S}}
\newcommand{\calA}{\mathcal{A}}
\newcommand{\cB}{\mathcal{B}}

\newcommand{\cE}{\mathcal{E}}

\newcommand{\cG}{\mathcal{G}}

\newcommand{\cI}{\mathcal{I}}

\newcommand{\cN}{\mathcal{N}}

\newcommand{\cO}{\mathcal{O}}

\newcommand{\cS}{\mathcal{S}}

\newcommand{\cV}{\mathcal{V}}
\newcommand{\cW}{\mathcal{W}}

\newcommand{\bbE}{\mathbb{E}}

\newcommand{\bbI}{\mathbb{I}}

\newcommand{\bbN}{\mathbb{N}}
\newcommand{\N}{\mathbb{N}}

\newcommand{\bbP}{\mathbb{P}}

\newcommand{\bbR}{\mathbb{R}}
\newcommand{\R}{\mathbb{R}}

\DeclareMathOperator*{\argmin}{argmin}
\DeclareMathOperator*{\argmax}{argmax}

\newcommand{\interior}{\mathrm{int}}

\newcommand{\dom}{\mathrm{dom}}
\newcommand{\sign}{\mathrm{sign}}
\newcommand{\fen}{Fenchel}
\newcommand{\leg}{Legendre}
\newcommand{\breg}{Bregman divergence}

\newcommand{\lrb}[1]{\left(#1\right)}
\newcommand{\brb}[1]{\bigl(#1\bigr)}
\newcommand{\Brb}[1]{\Bigl(#1\Bigr)}

\newcommand{\Bbrb}[1]{\Biggl(#1\Biggr)}
\newcommand{\lsb}[1]{\left[#1\right]}
\newcommand{\bsb}[1]{\bigl[#1\bigr]}
\newcommand{\Bsb}[1]{\Bigl[#1\Bigr]}
\newcommand{\Bbsb}[1]{\Biggl[#1\Biggr]}

\newcommand{\bcb}[1]{\bigl\{#1\bigr\}}

\newcommand{\labs}[1]{\left\lvert#1\right\rvert}

\newcommand{\lno}[1]{\left\lVert#1\right\rVert}

\newcommand{\lan}[1]{\left\langle#1\right\rangle}
\newcommand{\ban}[1]{\bigl\langle#1\bigr\rangle}

\newcommand{\s}{\subset}
\newcommand{\m}{\setminus}
\newcommand{\iop}{\infty}
\newcommand{\ld}{\ldots}
\renewcommand{\l}{\ldots}

\newcommand{\xt}{\widetilde{x}}

\newcommand{\diam}{\mathrm{diam}}
\newcommand{\Int}{\mathrm{int}}
\newcommand{\feedb}{f}

\newcommand{\conv}{\mathrm{co}}
\newcommand{\lhat}{\widehat{\ell}}

\newcommand{\xbar}{\overline{x}}

\newcommand{\Rbar}{\R\cup\{+\iop\}}
\newcommand{\coa}{\conv(\calA)}
\newcommand{\ExpCoop}{Exp3-COOP}
\newcommand{\coopftpl}{Coop-FTPL}
\newcommand{\coopgr}{Coop-GR}
\newcommand{\osmd}{OSMD}

\newtheorem{theorem}{Theorem}
\newtheorem{lemma}{Lemma}
\newtheorem{definition}{Definition}
\newtheorem*{definition*}{Definition}
\newtheorem{proposition}{Proposition}
\newtheorem*{proposition*}{Proposition}
\newtheorem{corollary}{Corollary}
\newtheorem*{corollary*}{Corollary}
\newtheorem*{theorem*}{Theorem}
\newtheorem*{lemma*}{Lemma}

\begin{document}

\maketitle

\begin{abstract}
We consider the problem of asynchronous online combinatorial optimization on a network of communicating agents. 
At each time step, some of the agents are stochastically activated, requested to make a prediction, and the system pays the corresponding loss.
Then, neighbors of active agents receive semi-bandit feedback and exchange some succinct local information. 
The goal is to minimize the network regret, defined as the difference between the cumulative loss of the predictions of active agents and that of the best action in hindsight, selected from a combinatorial decision set. 
The main challenge in such a context is to control the computational complexity of the resulting algorithm while retaining minimax optimal regret guarantees.
We introduce \coopftpl{}, a cooperative version of the well-known Follow The Perturbed Leader algorithm, that implements a new loss estimation procedure generalizing the Geometric Resampling of \cite{neu2013efficient} to our setting.
Assuming that the elements of the decision set are $k$-dimensional binary vectors with at most $m$ non-zero entries and $\alpha_1$ is the independence number of the network, we show that the expected regret of our algorithm after $T$ time steps is of order
$Q\sqrt{mkT\log(k) (k\alpha_1/Q+m)}$,  where $Q$ is the total activation probability mass. 
Furthermore, we prove that this is only $\sqrt{k\log k}$-away from the best achievable rate and that \coopftpl{} has a state-of-the-art $T^{3/2}$ worst-case computational complexity.
\end{abstract}

\section{Introduction}

Distributed online settings with communication constraints arise naturally in large-scale learning systems. 
For example, in domains such as finance or online advertising, agents often serve high volumes of prediction requests and have to update their local models in an online fashion. 
Bandwidth and computational constraints may therefore preclude a central processor from having access to all the observations from all sessions and synchronizing all local models at the same time.
With these motivations in mind, we introduce and analyze a new online learning setting in which a network of agents solves efficiently a common nonstochastic combinatorial semi-bandit problem by sharing information only with their network neighbors.
At each time step $t$, some agents $v$ belonging to a communication network $\cG$ are asked to make a prediction $x_t(v)$ belonging to a subset $\calA$ of $\{0,1\}^k$ and pay a (linear) loss $\ban{ x_t(v) , \ell_t }$ where $\ell_t\in [0,1]^k$ is chosen adversarially by an oblivious environment. 
Then, any such agent $v$ receives the feedback $\brb{ x_{t}(1,v)\ell_{t}(1), \ld ,x_{t}(k,v)\ell_{t}(k) }$, which is shared, together with some local information, to its first neighbors in the graph. 
The goal is to minimize the network regret after $T$ time steps
\begin{equation}
\label{e:regret-intro}
    R_T
=
	\max_{a \in \calA} \bbE \lsb{ \sum_{t=1}^T \sum_{v\in \cA_t} \ban{x_t(v), \ell_t} - \sum_{t=1}^T \sum_{v\in \cA_t} \lan{a, \ell_t} }\;,
\end{equation}
where $\cS_t$ is the set of agents $v$ that made a prediction at time $t$.
In words, this is the difference between the cumulative loss of the ``active'' agents and the loss that they would have incurred had they consistently made the best prediction in hindsight. 

For this setting, we design a distributed algorithm that we call \coopftpl{} (Algorithm~\ref{FTPL}), and prove that its regret is upper bounded by $Q\sqrt{mkT\log(k) ({k\alpha_1}/{Q}+m)}$ (Theorem~\ref{thm:main}), where $\alpha_1$ is the independence number of the network $\cG$ and $Q$ is the sum over all agents of the probability that the agent is active during a time step. 
Our algorithm employs an estimation technique that we call Cooperative Geometric Resampling (\coopgr{}, Algorithm~\ref{GRcoop}). 
It is an extension of a similar procedure appearing in \citep{neu2013efficient} that relies on the fact that the reciprocal of the probability of an event can be estimated by measuring the reoccurrence time. 
We can leverage this idea in the context of cooperative learning thanks to some statistical properties of the minimum of a family of geometric random variables (see Lemmas~\ref{lem:Min-Geoms}--\ref{lem:buildGeoms}).
Our algorithm has a state-of-the-art dependence on time of order $T^{3/2}$ for the worst-case computational complexity (Proposition~\ref{p:comput-complexity}).
Moreover, we show with a lower bound (Theorem~\ref{t:lower-bound}) that no algorithm can achieve a regret smaller than $Q\sqrt{m k T\alpha_1/Q}$ on all cooperative semi-bandit instances. 
Thus, our \coopftpl{} is at most a multiplicative factor of $\sqrt{k\log k}$-away from the minimax result.

To the best of our knowledge, ours is the first computationally efficient near-optimal learning algorithm for the problem of cooperative learning with nonstochastic combinatorial bandits, where not all agents are necessarily active at all time steps.

\section{Related work and further applications}

Single-agent combinatorial bandits find applications in several fields, such as path planning, ranking and matching problems, finding minimum-weight spanning trees, cut sets, and multitask bandits.
An efficient algorithm for this setting is Follow-The-Perturbed-Leader (FTPL), which was first proposed by \cite{hannan1957approximation} and later rediscovered by \cite{kalai2005efficient}. 
\cite{neu2013efficient} show that combining FTPL with a loss estimation procedure called Geometric Resampling (GR) leads to a computationally efficient solution for this problem.
More precisely, the solution is efficient given that the offline optimization problem of finding
\begin{equation}\label{eq:efficientOPT}
a^{\star}=\argmin_{a \in \mathcal{A}}\langle a, y\rangle\;, \qquad \forall y\in[0,+\infty)^k
\end{equation}
admits a computationally efficient algorithm. 
This assumption is minimal, in the sense that if the offline problem in Eq.~\eqref{eq:efficientOPT} is hard to approximate, then any algorithm with low regret must also be inefficient.\footnote{A slight relaxation in this direction would be assuming that Eq.~\eqref{eq:efficientOPT} can be approximated accurately and efficiently.}
\cite{grotschel2012geometric} and \cite{lee2018efficient} give some sufficient conditions for the validity of this assumption. 
They essentially rely on having an efficient membership oracle for the convex hull $\conv(\calA)$ of $\calA$ and an evaluation oracle for the linear function to optimize.
\cite{audibert2014regret} note that Online Stochastic Mirror Descent (OSMD) or Follow The Regularized Leader (FTRL)-type algorithms can be efficiently implemented by convex programming if the convex hull of the decision set can be described by a polynomial number of constraints.
\cite{suehiro2012online} investigate the details of such efficient implementations and design an algorithm with $k^{6}$ time-complexity, which might still be unfeasible in practice. 
Methods based on the exponential weighting of each decision vector can be implemented efficiently only in a handful of special cases ---see, e.g., \citep{koolen2010hedging} and \citep{cesa2012combinatorial} for some examples. 

The study of cooperative nonstochastic online learning on networks was pioneered by \cite{awerbuch2008competitive}, who investigated a bandit setting in which the communication graph is a clique, agents belong to clusters characterized by the same loss, and some agents may be non-cooperative. 
In our multi-agent setting, the end goal is to control the total network regret \eqref{e:regret-intro}.
This objective was already studied by \cite{cesa2019cooperative} in the full-information case. 
A similar line of work was pursued by \cite{cesa2019delay}, where the authors consider networks of learning agents that cooperate to solve the same nonstochastic bandit problem. 
In their setting, all agents are simultaneously active at all time steps, and the feedback propagates throughout the network with a maximum delay of $d$ time steps, where $d$ is a parameter of the proposed algorithm. 
The authors introduce a cooperative version of Exp3 that they call \ExpCoop{} with regret of order $\sqrt{(d+1+{K\alpha_{d}}/{N}) (T\log K)}$  where $K$ is the number of arms in the nonstochastic bandit problem, $N$ is the total number of agents in the network, and $\alpha_{d}$ is the independence number of the $d$-th power of the communication network. 
The case $d=1$ corresponds to information that arrives with one round of delay and communication limited to first neighbors.
In this setting \ExpCoop{} has regret of order $\sqrt{(1+{K\alpha_1}/{N}) (T\log K)}$. Thus, our work can be seen as an extension of this setting to the case of combinatorial bandits with stochastic activation of agents. 
Finally, we point out that if the network consists of a single node, our cooperative setting collapses into a single-agent combinatorial semi-bandit problem. 
In particular, when the number of arms is $k$ and $m=1$, this becomes the well-known adversarial multiarmed bandit problem  (see \citep{auer2002finite}).
Hence, ours is a proper generalization of all the settings mentioned above.
These cooperative problems are also studied in stochastic setting (see, e.g., \cite{martinez2019decentralized}).

Finally, the reader may wonder what kind of results could be achieved if the agents are activated adversarially rather than stochastically.
\cite{cesa2019cooperative} showed that in this setting no learning can occur, not even in with full-information feedback.

\section{Cooperative semi-bandit setting}
In this section, we present our cooperative semi-bandit protocol and we introduce all relevant definitions and notation.

We say that $\cG=(\cV, \cE)$ is a \emph{communication network} over $N$ agents if it is an undirected graph over a set $\cV$ with cardinality $\labs{ \cV } = N$, whose elements we refer to as \emph{agents}. 
Without loss of generality, we assume that $\cV=\{1, \ldots, N\}$. 
For any agent $v \in \cV$, we denote by $\cN(v)$ the set of agents containing $v$ and the neighborhood $\{w \in \cV : (v, w) \in \cE\}$. 
We say that $\alpha_1$ is the \emph{independence number} of the network $\cG$ if is the largest cardinality of an \emph{independent set} of $\cG$, where an independent set of $\cG$ is a subset of agents, no two of which are neighbors.

We study the following cooperative online combinatorial optimization protocol. 
Initially, hidden from the agents, the environment draws a sequence of subsets $\cS_{1}, \cS_{2}, \ldots \s \cV$ of agents, that we call \emph{active}, and a sequence of  \emph{loss vectors} $\ell_1, \ell_2, \ldots \in \bbR^k$.
We assume that each agent $v$ has a probability $q(v)$ of being activated, which need only be known by $v$.
The set of active agents $\cS_t$ at time $t\in \{1,2,\ld\}$ is then determined by drawing, for each agent $v\in \cV$, a Bernoulli random variable $X_t(v)$ with bias $q(v)$, independently of the past, and $\cS_t$ consists exclusively of agents $v \in \cV$ for which $X_t(v)=1$.
The \emph{decision set} is a subset $\calA$ of $\bcb{ a\in \{0,1\}^k : \sum_{i=1}^k a_i \leq m }$, for some $m \in \{1,\ld,k\}$.

\noindent{}For each time step $t \in\{1,2,\ldots\}$:
\begin{enumerate}[topsep = 0pt, parsep = 0pt, itemsep = 0pt]
\item each active agent $v \in \cS_{t}$ predicts with $x_{t}(v) \in \calA$ (possibly drawn at random);
\item each neighbor $v \in \cN(w)$ of an active agent $w \in \cS_t$ receives the feedback
\begin{equation}
\label{e:feedb}
    \feedb_t(w) := \brb{ x_{t}(1,w)\ell_{t}(1), \ld ,x_{t}(k,w)\ell_{t}(k) }\;;
\end{equation}
\item each agent $v \in \bigcup_{w \in \cS_{t}} \cN(w)$ receives some local information from its neighbors in $\cN(v)$;
\item the system incurs the loss $\sum_{v\in \cA_t} \ban{ x_t(v), \ell_t }$.
\end{enumerate}
The goal is to minimize the expected \emph{network regret} as a function of the \emph{time horizon} $T$, defined by
\begin{equation}\label{eq:totalExpRegret}
    R_T
:=
    \max_{a \in \calA} \bbE\lsb{
	\sum_{t=1}^T \sum_{v\in \cA_t} \ban{ x_t(v), \ell_t }
	- \sum_{t=1}^T \sum_{v\in \cA_t} \lan{ a, \ell_t }
	}\;,
\end{equation}
where the expectation is taken with respect to the draws of $\cS_1, \ldots, \cS_T$ and (possibly) the randomization of the learners.
In the next sections we will also denote by $\bbP_t$ the probability conditioned to the history up to and including round $t-1$, and by $\bbE_t$ the corresponding expectation.

The nature of the local information exchanged by neighbors of active agents will be clarified in the next section. 
In short, they share succinct representations of the current state of their local prediction model.

\section{\coopftpl{} and upper bound} 
\label{s:upper-bound}

In this section we introduce and analyze our efficient \coopftpl{} algorithm for cooperative online combinatorial optimization.

\subsection{The Coop-FTPL algorithm}

\coopftpl{} (Algorithm~\ref{FTPL}) takes as input a decision set $\calA \s \{0,1\}^k$, a time horizon $T\in \N$, a learning rate $\eta > 0$, a truncation parameter $\beta$, and an exploration distribution $\zeta$.
At each time step $t$, all active agents $v$ make a FTPL prediction $x_t(v)$ (line~\ref{state:ftpl-update}) with an i.i.d.\ perturbation sampled from $\zeta$ (line~\ref{state:perturb}), then they receive some feedback $f_t(v)$ and share it with their first neighbors (line~\ref{state:feedb}).
Afterwards, each agent who received some feedback this round, requests from its neighbors $w$ a vector $K_t(w)$ of geometric random samples (line~\ref{state:local-info}) which is efficiently computed by Algorithm~\ref{GRcoop} and will be described in detail later.
With these geometric samples, each agent $v$ computes an estimated loss $\widehat{\ell}_t(v)$ (line~\ref{state:update-ell}) and updates the cumulative loss estimate $\widehat{L}_t(v)$ (line~\ref{state:update-L}).
This estimator is described in detail in Section~\ref{s:esti}.

\begin{algorithm2e}
\caption{\label{FTPL}Follow the perturbed leader for cooperative semi-bandits (\coopftpl{})}
    \LinesNumbered
    \SetAlgoNoLine
    \SetAlgoNoEnd
    \DontPrintSemicolon
    \SetKwInput{kwInit}{Initialization}
    \KwIn{decision set $\calA$, horizon $T$, learning rate $\eta$, truncation parameter $\beta$, exploration pdf $\zeta$}
    \kwInit{$\widehat{L}_{0} = 0 \in \mathbb{R}^{k}$}
\For{each time $t=1,2,\l$}
{
    \For(\tcp*[f]{active agents}){
           each active agent $v\in \cS_t$
    }
    {
        sample $Z_{t}(v) \sim \zeta$, independently of the past\label{state:perturb}\;
        make the prediction $x_{t}(v)=\argmax_{a \in \mathcal{A}}\ban{ a, Z_{t}(v)-\eta \widehat{L}_{t-1}(v) }$\label{state:ftpl-update}\;
    }
    \For(\tcp*[f]{neighbors of active agents}){each agent $v\in \bigcup_{w \in \cS_t}\cN(w)$} 
    {
        receive the feedback $\feedb_{t}(w)$ (Eq.~\eqref{e:feedb}) from each active neighbor $w \in \cN(v) \cap \cS_t$\label{state:feedb}\;
        receive $K_t(i,w)$ ($\forall \ i \in \{1,\ld,k\}$) from each neighbor $w\in \cN(v)$ using Algorithm~\ref{GRcoop}\label{state:local-info}\;
        compute $\widehat{\ell}_t (i,v)=\ell_{t}(i) \, B_{t} \, (i,v) \min_{w\in\cN (v)} \bcb{ K_t(i,w)}$, $\forall i\in\{1,\ldots,k\}$ (Eq.~\eqref{e:c-1}--\eqref{e:def-B})\label{state:update-ell}\;
        update $\widehat{L}_{t}(v)=\widehat{L}_{t-1}(v)+\widehat{\ell}_t (v)$\label{state:update-L}
    }
}
\end{algorithm2e}

\begin{algorithm2e}
    \LinesNumbered
    \SetAlgoNoLine
    \SetAlgoNoEnd
    \DontPrintSemicolon
    \KwIn{time step $t$, component $i$, agent $w$, truncation parameter $\beta$, exploration pdf $\zeta$}
\For{$s = 1,2,\ld$}
{
    sample $z'_s \sim \zeta$, independently of the past\;
    let $x'_s$ be the $i$-th component of the vector $\argmax_{a\in\calA} \ban{ a,z'_s-\eta\widehat{L}_{t-1}(w)}$\;
    sample $y'_s$ from a Bernoulli distribution with parameter $q(w)$, independently of the past\;
    if $( x'_s=1 \textbf{ and } y'_s=1 ) \textbf{ or } s=\beta$, \textbf{break}\;
}
\Return $K_t(i,w) = s$\;
\caption{\label{GRcoop} Geometric resampling for cooperative semi-bandits (\coopgr{}) \label{GeomResampling}}
\end{algorithm2e}

\subsection{Reduction to OSMD}
\label{s:redu-osmd}

Before describing $K_t(w)$ and $\widehat{\ell}_t(v)$, we make a connection between FTPL and the Online Stochastic Mirror Descent algorithm (OSMD)\footnote{For a brief overview of some key convex analysis and OSMD facts, see Appendices~\ref{s:convex-recap}~and~\ref{s:osmd}} that will be crucial for our analysis.\footnote{For a similar approach in the single-agent case, see \citep{lattimore2020bandit}.}

Fix any time step $t$ and an agent $v$. 
As we mentioned above, if $v$ is active, it makes the following FTPL prediction (line~\ref{state:ftpl-update})
\[
    x_{t}(v)
=
    \argmin_{a\in\calA} \ban{ a,\eta\widehat{L}_{t-1}(v)-Z_{t}(v)} \;,    
\]
where $Z_{t}(v)\in\mathbb{R}^{k}$ is sampled i.i.d.\ from $\zeta$ (the random perturbations introduce the exploration, which for an appropriate
choice of $\zeta$ is sufficient to guarantee small regret).
On the other hand, given a Legendre potential $F$ with $\text{dom}\left(\nabla F\right)=\interior\left(\text{co}\left(\mathcal{A}\right)\right)$,
an \osmd{} algorithm makes the prediction
\[
    \xbar_t(v)
=
    \argmin_{x\in\conv(\calA)} \Brb{
    \ban{ x, \eta \, \widehat{\ell}_{t-1}(v) }
    + \cB_{F} ( x, \xbar_{t-1}(v) )
    } \;,
\]
where $\cB_F$ is the \breg{} induced by $F$ and $\conv(\calA)$ is the convex hull of $\calA$.
Using the fact that $\text{dom}(\nabla F) = \interior\brb{ \conv(\calA) }$, the $\argmin$ above can be computed in a standard way by studying when the gradient of its argument is equal to zero, and proceeding inductively, we obtain the two identities
$
    \nabla F(\xbar_t(v))
=
    \nabla F(\xbar_{t-1}(v))-\eta\, \widehat{\ell}_{t-1}(v)
=
    -\eta\widehat{L}_{t-1}(v) \;.
$
By duality this implies that $\xbar_t(v)=\nabla F^{*}\brb{ -\eta\widehat{L}_{t-1}(v)}$.
We now want to relate $x_t(v)$ and $\xbar_t(v)$ so that
\begin{equation}
\label{e:osmd-ftlp}
    \xbar_t(v)
=
    \bbE_t \bsb{ x_t(v) }
=
    \bbE_t \lsb{
        \argmin_{a\in\calA}
        \ban{ a,\eta\widehat{L}_{t-1}(v)-Z_{t}(v)}
    }\;,
\end{equation}
where the conditional expectation $\bbE_t$ (given the history up to time $t-1$) is taken with respect to $Z_t(v)$.
Thus, in order to view FTPL as an instance of \osmd{}, it suffices
to find a Legendre potential $F$ with $\text{dom}\left(\nabla F\right)=\text{int}\left(\text{co}\left(\mathcal{A}\right)\right)$
such that
$
    \nabla F^{*}\brb{ -\eta\widehat{L}_{t-1}(v) }
=
    \mathbb{E}_t \bsb{
    \argmax_{a\in \calA}
    \ban{ 
        a,Z_{t}(v)-\eta\widehat{L}_{t-1}(v)
    }
    }
$.
In order to satisfy this condition, we need that for any $x\in\mathbb{R}^{k}$,
the \fen{} conjugate $F^*$ of $F$ enjoys
$\nabla F^{*}\left(x\right)=\int_{\mathbb{R}^{k}}\argmax_{a\in\mathcal{A}}\left\langle a,z-x\right\rangle \:\zeta(z)\:\mathrm{d}z.$
Then, we define
$
    h(x)
:=
    \argmax_{a\in\mathcal{A}}\left\langle a,x\right\rangle 
$ 
for any $x\in\bbR^k$,
where $h(x)$ is chosen to be an arbitrary maximizer if multiple maximizers
exist. 
From convex analysis, if the convex hull $\conv(\mathcal{A})$ of $\calA$ had a
smooth boundary, then the support function
$
    x\mapsto \phi(x)
:=
    \max_{a\in\conv(\mathcal{A})}\left\langle a,x\right\rangle
=
    \max_{a\in\mathcal{A}}\left\langle a,x\right\rangle ,
$
of $\conv(\mathcal{A})$ would satisfy $\nabla\phi(x)=h(x)$. 
For combinatorial bandits, $\conv(\mathcal{A})$
is non-smooth, but, being $\zeta$ a density with respect to Lebesgue
measure, one can prove (see, e.g., 
\cite{lattimore2020bandit}) 
that
$
    \nabla\int_{\mathbb{R}^{k}}\phi\left(x+z\right)\zeta(z)\:\mathrm{d}z
=
    \int_{\mathbb{R}^{k}} h\left(x+z\right)\zeta(z)\:\mathrm{d}z
$,
for all $x\in\mathbb{R}^{k}$.
This shows that FTPL can be interpreted as \osmd{} with
a potential $F$ defined implicitly by its Fenchel conjugate
\[
    F^{*}(x)
:=
    \int_{\mathbb{R}^{k}}\phi\left(x+z\right)\zeta(z)\:\mathrm{d}z\;, \qquad \forall x \in \bbR^k\;.
\]
Thus, recalling \eqref{e:osmd-ftlp}, we can think of the update $\xbar_t(v)$ of \osmd{} as the average of a random component-wise draw 
$\xbar_t(i,v) =\sum_{a\in\mathcal{A}}P_{t}(a,v) \, a(i)$ (for all $i\in\{1,\ld,k\}$),
with respect to a distribution $P_t(v)$ on $\calA$ defined in terms of the distribution of $Z_t$, as
\[
    P_t(a,v)
=
    \bbP_t
    \Bsb{
        h \brb{ Z_{t}(v)-\eta\widehat{L}_{t-1}(v) } = a
    }\;, \qquad \forall a \in \calA\;,
\]
where $\bbP_t$ is the probability conditioned of the history up to time $t-1$.

\subsection{An efficient estimator}
\label{s:esti}

For the understanding of the definitions and analyses of $K_t(w)$ and $\widehat{\ell}_t(v)$, we introduce three useful lemmas on geometric distributions.
We defer their proofs to Appendix~\ref{s:geometric-appe}.

\begin{lemma}
\label{lem:Min-Geoms}
Let $Y_1,\ld,Y_m$ be $m$ independent random variables such that each $Y_j$ has a geometric distribution with parameter $p_j \in [0,1]$. Then, the random variable $Z := \min_{j\in \{1,\ld,m\}} Y_j$ has a geometric distribution with parameter $1-\prod_{j=1}^m (1-p_j)$.
\end{lemma}

\begin{lemma}
\label{lem:Min-Geom-Beta}
Let $G$ be a geometric random variable with parameter $q\in(0,1]$ and $\beta > 0$. 
Then, the expectation of the random variable $\min\{G, \beta\}$ satisfies $\bbE\bsb{ \min\{G,\beta\} } = \brb{ 1 - (1-q)^\beta }/q$.
\end{lemma}

\begin{lemma}
\label{lem:buildGeoms}
For all $v \in \cV$, fix two arbitrary numbers $p_1(v), p_2(v) \in [0,1]$. 
Consider a collection $\bcb{ X_s (v), Y_s(v) }_{ s \in \bbN, v \in \cV }$ of independent Bernoulli random variables such that $\bbE \bsb{ X_s(v) } = p_1(v)$ and $\bbE \bsb{ Y_s(v) } = p_2(v)$ for any $s\in \bbN$ and all $v\in \cV$.
Then, the random variables $\{ G(v) \}_{v\in \cV}$ defined for all $v\in \cV$ by $G(v) := \inf\bcb{ s \in \bbN : X_s(v) \, Y_s(v) = 1 }$ are all independent and they have a geometric distribution with parameter $p_1(v)\,p_2(v)$.
\end{lemma}
Fix now any time step $t$, agent $v$, and component $i\in \{1,\ld,k\}$. 
The loss estimator $\widehat{\ell}_t (i,v)$ depends on the algorithmic definition of $K_t(i,w)$ in Algorithm~\ref{GRcoop}, where $w\in \cN(v)$.
By Lemma~\ref{lem:buildGeoms}, we have that for any $w$, conditionally on the history up to time $t-1$, the random variable $K_t(i,w)$, has a truncated geometric distribution with success probability equal to $\xbar_t(i,w)q(w)$ and truncation parameter $\beta$. 
The idea  is that using $K_t(i,w)$ as an estimator we can reconstruct the inverse of the probability that agent $v$ observes $i$ at time $t$.
The truncation parameter $\beta$ is not needed for the analysis, but it is used just to optimize the computational efficiency of the algorithm.\footnote{Previously known cooperative algorithms for limited feedback settings need to exchange at least two real numbers: the probability according to which predictions are drawn and the loss. Instead of the probability, we only need to pass the integer $K_t(i,w)$, which requires at most $\log_2(\beta)$ bits (order of $\log(T)$, when tuned). Note also that for the loss, we could exchange an approximation $l_t$ of $\ell_t$ using only $n = O(\log(T))$ bits. Indeed one can show that Lemma~\ref{lem:The-expectation-of}, in this case, is true replacing $\ell_t$ with $l_t$ in the first equality. Everything else works the same in the proof of Theorem~\ref{thm:main} up an extra (negligible) $m2^{-n}T$ term.}

The loss estimator of $v$ is then defined as
\begin{equation}
    \label{e:c-1}
    \widehat{\ell}_t (i,v)
:=
    \ell_{t}(i) \, B_{t} \, (i,v) \min_{w\in\cN (v)} \bcb{ K_t (i,w) } ,
\end{equation}
where 
\begin{equation}
\label{e:def-B}
    B_t(i,v)
=
    \bbI \bcb{ 
        \exists w\in\cN(v) : w\in\cS_t,\,x_t(i,w)=1 
    } \;,
\qquad
    K_t (i,w)
=
    \min\bcb{ G_t (i,w), \beta }\;,
\end{equation}
and given the history up to time $t-1$, for each $i\in\{1,\ldots,k\}$, the family $\bcb{ G_{t}(i,w) }_{ w\in\cV}$ consists of independent geometric random variables with parameter $\xbar_t(i,w)q(w)$.
Note that the geometric random variables $G_{t}(i,w)$ are actually never computed by Algorithm~\ref{GRcoop} which efficiently computes only their truncations $K_t(i,w)$, with truncation parameter $\beta$. 
Nevertheless, as it will be apparent later, they are a useful tool for the theoretical analysis.
Note that, by Eq.~\eqref{e:osmd-ftlp}, we have
\[
    \mathbb{P}_t\bsb{ x_{t}(i,w)=1 }
=
    \bbE_t[x_t(i,w)]
=
    \overline{x}_{t}\left(i,w\right),
\]
therefore
\[
    \overline{B}_{t}(i,v)
:=
    \mathbb{E}_{t}\left[B_{t}(i,v)\right]
=
    1-\prod_{w\in\cN(v)}\left(1-\xbar_{t}(i,w)\,q(w)\right)=\frac{1}{\mathbb{E}_{t}\left[\min_{w\in \mathcal N(v)}G_{t}(i,w)\right]}\;,
\]
where the last identity follows by Lemma~\ref{lem:Min-Geoms}. 
Moreover from Lemmas~\ref{lem:Min-Geoms}~and~\ref{lem:Min-Geom-Beta}, we have
\[
\mathbb{E}_{t}\left[\min_{w\in\cN(v)}K_{t}(i,w)\right]=\frac{1-\prod_{w\in\cN(v)}\left(1-\overline{x}_{t}(i,w)\,q(w)\right)^{\beta}}{\overline{B}_{t}(i,v)}\;.
\]
The following key lemma gives an upper bound on the expected estimated loss.
\begin{lemma}
\label{lem:The-expectation-of}
For any time $t$, component $i$, agents $v$, and truncation parameter $\beta$, the expectation of the loss estimator in \eqref{e:c-1}, given the history up to time $t-1$, satisfies
\[
    \bbE_t \Bsb{ \widehat{\ell}_{t}(i,v) }
=
    \ell_{t}(i)\lrb{ 
        1-\Bbrb{ \prod_{w\in\cN(v)} \brb{ 1-\xbar_t (i,w) \, q(w) } }^\beta
    }
\le
    \ell_t(i) \;.
\]
\end{lemma}
\begin{proof}
Using the fact that, conditioned on the history up to time $t-1$, the random variable $\min_{w\in\cN(v)}G_{t}(i,w)$ has a geometric distribution with parameter $\overline{B}_{t}(i,v)$
(Lemmas \ref{lem:Min-Geoms}-\ref{lem:buildGeoms}), we get
\begin{align*}
&\mathbb{E}_{t}\left[\widehat{\ell}_{t}(i,v)\right] 
\\
&\ =
    \mathbb{E}_{t}\left[\ell_{t}(i)B_{t}(i,v)\min_{w\in\cN(v)}\left\{ \min\left\{ G_{t}(i,w),\beta\right\} \right\} \right]
=
    \mathbb{E}_{t}\left[\ell_{t}(i)B_{t}(i,v)\min\left\{ \min_{w\in\cN(v)}G_{t}(i,w),\beta\right\} \right]
\\
&\ =
    \ell_{t}(i) \mathbb E_t\left[B_{t}(i,v)\right]\mathbb{E}_{t}\left[\min\left\{ \min_{w\in\cN(v)}G_{t}(i,w),\beta\right\} \right]
=
    \ell_{t}(i)\overline{B}_{t}(i,v)\frac{\left(1-\left(1-\overline{B}_{t}(i,v)\right)^{\beta}\right)}{\overline{B}_{t}(i,v)}
\\
&\ =
    \ell_{t}(i)\left(1-\left(1-\overline{B}_{t}(i,v)\right)^{\beta}\right)
=
    \ell_{t}(i)
        \lrb{ 
            1-\Bbrb{
                \prod_{w\in\cN(v)} \brb{ 1-\xbar_{t}(i,w)\,q(w) } 
            }^{\beta}
        }
        \;,
\end{align*}
where we plugged in the definition of $\overline{B}_{t}(i,v)$ in the last
equation. 
From the fact that $\xbar_{t}(i,w)\,q(w)\in[0,1]$ and $\beta>0$ follows that $\mathbb{E}_{t}\left[\widehat{\ell}_{t}(i,v)\right]\leq\ell_{t}(i).$
\end{proof}

\subsection{Analysis}

We can finally state our upper bound on the regret of \coopftpl{}.
The key idea is to apply OSMD techniques to our FTPL algorithm, as explained in Section~\ref{s:redu-osmd}.
The proof proceeds by splitting the regret of each agent in the network in three terms.
The first two are treated with standard techniques; the first one depends on the diameter of $\conv(\calA)$ with respect to the regularizer $F$ and the second one on the Bregman divergence of consecutive updates.
The last term is related to the bias of the estimator and is analyzed leveraging the lemmas on geometric distributions introduced in Section~\ref{s:esti}.
Then, this terms are summed, each with a weight corresponding to their probabilities of being activated, and this sum is controlled using results that relate a sum of weights over the nodes of a graph of with the independence number of the graph.

\begin{theorem}
\label{thm:main}
If $\zeta$ is the Laplace density $z\mapsto\zeta(z)=2^{-k}\exp\brb{-\lno{z}_1 }$, $\eta >0$, and $0 < \beta \le 1/ (\eta k)$, then then the regret of our \coopftpl{} algorithm (Algorithm~\ref{FTPL}) satisfies
\[
 R_T \leq
    \frac{3 Q m \log k }{\eta}
    + 3 Q \eta kT\left(\frac{k}{Q}\alpha_1+m\right)
    +\frac{\alpha_1 \, k \, T}{\beta} \;.
\]
In particular, tuning the parameters $\eta,\beta$ as follows 
\begin{equation}
\label{e:tuning}
\beta = \left\lfloor \frac{1}{k \eta}\right\rfloor
\quad\text{and}\quad
    \eta
=
    \sqrt{\frac{3m\log(k)}{5kT\left(\frac{k}{Q}\alpha_1+m\right)}}\;,
    \qquad\text{where }
    \qquad Q
=
    \sum_{v\in\mathcal{V}}q(v)\;,
\end{equation}
yields
\[
    R_T
\le
    2Q\sqrt{ 15 mkT\log(k)\left(\frac{k}{Q}\alpha_1+m\right)} \;.
\]
\end{theorem}
We now present a detailed sketch of the proof of our main result
(full proof in Appendix~\ref{s:proof-main-thm}).
\begin{proof}[Sketch of the proof]
For the sake of convenience, we define the expected individual regret of an agent $v\in \cV$ in the network with respect to a fixed action $a \in \calA$ by
\[
    R_T(a,v)
:=
    \mathbb{E}\left[\sum_{t=1}^{T}\left\langle x_{t}(v),\ell_{t}\right\rangle-\sum_{t=1}^{T}\left\langle a,\ell_{t}\right\rangle \right]\;,
\]
where the expectation is taken with respect to the internal randomization of the agent, but not to its activation probability $q(v)$. With this definition the total regret on the network in Eq.~\eqref{eq:totalExpRegret} can be decomposed as 
\begin{align}
    R_T \nonumber
&=
    \max_{a \in \calA}\bbE\lsb{\sum_{t=1}^{T}\sum_{v\in \cA_t}\Brb{ \ban{ x_{t}(v),\ell_{t} } -\left\langle a,\ell_{t}\right\rangle}}
=
    \max_{a \in \calA}\bbE\lsb{\sum_{t=1}^{T}\bbE_t\lsb{\sum_{v\in \cA_t}\Brb{ \ban{ x_{t}(v),\ell_{t} } -\left\langle a,\ell_{t}\right\rangle}}}\\
&=    
    \max_{a \in \calA}\bbE\lsb{\sum_{t=1}^{T}\sum_{v\in\cV} q(v) \bbE_t\Bsb{\ban{ x_{t}(v),\ell_{t} }-\left\langle a,\ell_{t}\right\rangle}}
=    
    \max_{a \in \calA}\sum_{v\in\cV} q(v) R_T(a,v)\;.\label{eq:decompositionRegret_sketch}
\end{align}
The proof then proceeds by isolating the bias in the loss estimators. 
For each $a\in \calA$ we have
\begin{equation*}
    R_{T}(a,v) 
 =
    \mathbb{E}\left[\sum_{t=1}^{T}\left\langle \xbar_t(v)-a,\widehat{\ell}_{t}(v)\right\rangle \right]+\mathbb{E}\left[\sum_{t=1}^{T}\left\langle \xbar_t(v)-a,\ell_{t}-\widehat{\ell}_{t}(v)\right\rangle \right]\;.
\end{equation*}
Exploiting the analogy that we established between FTPL and OSMD, we begin by using the standard bound for the regret of OSMD in the first term of the previous equation. 
For the reader's convenience, we restate it in Appendix~\ref{s:osmd}, Theorem~\ref{thm:regretOSMD}. 
This leads to
\begin{equation*}
    R_{T}(a,v) \! \le \!
    \underbrace{\frac{F(\xbar_{1}(v))-F(a)}{\eta}}_{(\mathrm{I})}
    \! + \underbrace{\mathbb{E}\left[\frac{1}{\eta}\sum_{t=1}^{T}\cB_{F}\left(\xbar_t(v),\xbar_{t+1}(v)\right)\right]}_{(\mathrm{II})}
    \! + \underbrace{\mathbb{E}\left[\sum_{t=1}^{T}\left\langle \xbar_t(v)-a,\ell_{t}-\widehat{\ell}_{t}(v)\right\rangle \right]}_{(\mathrm{III})}.
\end{equation*}
The three terms are studied separately and in detail in Appendix~\ref{app:proof}. 
Here, we provide a sketch of the bounds. 

For the first term $(\mathrm{I})$, we use the fact that the regularizer $F$ satisfies, for all $a\in\calA$, 
\begin{align}
F(a) \geq-m\brb{ 1+\log(k) },\label{eq:combinatorialDiam_sketch}
\end{align}
which follows by the definition of $F$, the properties of the perturbation distribution, and the fact that $\lno{ a}_{1}\leq m$ for any $a\in\mathcal{A}$. 
One can also show that $F(a)\leq0$ for all $a\in\mathcal{A}$, and this, combined with the previous equation, leads to
\[
    (\mathrm{I})
\le
    \frac{m (1+\log k)}{\eta} \;.
\]
For the second term $(\mathrm{II}),$ we have 
\begin{align}
    \cB_{F}\left(\xbar_{t}(v),\xbar_{t+1}(v)\right) 
& =
    \cB_{F^{*}}\left(\nabla F\left(\xbar_{t+1}(v)\right),\nabla F\left(\xbar_t(v)\right)\right)\nonumber \\
 & =
    \cB_{F^{*}}\left(-\eta\widehat{L}_{t-1}(v)-\eta\widehat{\ell}_{t}(v),-\eta\widehat{L}_{t-1}(v)\right)\nonumber\\
 & =
    \frac{\eta^{2}}{2}\left\Vert \widehat{\ell}_{t}(v)\right\Vert _{\nabla^{2}F^{*}(\xi(v))}^{2}\;,\label{eq:divereCOmb_sketch}
\end{align}
where the first equality is a standard property of Bregmann divergence (see Theorem~\ref{t:fenchel-stuff} in Appendix~\ref{s:convex-recap}),
the second follows from the definitions of the updates and the last
by Taylor's theorem, where 
$\xi(v)=-\eta \widehat{L}_{t-1}(v)-\alpha \eta \widehat{\ell}_{t}(v)$, for some $\alpha \in [0,1]$.
The estimation of the entries of the Hessian are non trivial (but tedious); the interested reader can find them in Appendix~\ref{app:proof}. 
Exploiting our assumption that $\beta \le 1 /(\eta k)$, we get, for all $i,j\in\{1,\ldots,k\}$,
\begin{align*}
    \nabla^{2} F^{*} \brb{ \xi(v) }_{i j} 
& \leq
    e \, \xbar_{t}(i,v)\;.
\end{align*}
Plugging this estimate in Eq.~\eqref{eq:divereCOmb_sketch} yields
\begin{align*}
\frac{\eta^{2}}{2}\left\Vert \widehat{\ell}_{t}(v)\right\Vert _{\nabla^{2}F^{*}(\xi(v))}^{2} 
& =
    \frac{\eta^{2}}{2}\sum_{i=1}^{k}\sum_{j=1}^{k}\nabla^{2}F^{*}\brb{ \xi(v) }_{i,j}\widehat{\ell}_{t}(i,v)\widehat{\ell}_{t}(j,v)\\
 & \leq
    \frac{\eta^{2} e}{2}\sum_{i=1}^{k}\sum_{j=1}^{k}\xbar_{t}(i,v)\widehat{\ell}_{t}(i,v)\widehat{\ell}_{t}(j,v)\\
 & \leq
    \frac{\eta^{2}e}{2}\sum_{i=1}^{k}\sum_{j=1}^{k}\xbar_{t}(i,v)B_{t}(i,v)\min_{w\in\cN(v)}\left\{ G_{t}(i,w)\right\} B_{t}(j,v)\min_{w\in\cN(v)}\left\{ G_{t}(j,w)\right\} ,
\end{align*}
where the last inequality follows by neglecting the truncation with
$\beta$. Hence multiplying $(\mathrm{II})$ by $q(v)$ and summing over $v\in \cV$ yields
\begin{align*}
& 
    \sum_{v\in\mathcal{V}}q(v)\mathbb{E}\left[\frac{\eta}{2}\sum_{t=1}^{T}\left\Vert \widehat{\ell}_{t}(v)\right\Vert _{\nabla^{2}F^{*}(\xi(v))}^{2}\right]
=
    \sum_{v\in \cV}q(v)\frac{\eta }{2}\mathbb{E}\left[\sum_{t=1}^{T}\mathbb{E}_{t}\left[\left\Vert \widehat{\ell}_{t}(v)\right\Vert _{\nabla^{2}F^{*}(\xi(v))}^{2}\right]\right]
\\
& \hspace{2.94753pt} \leq
    \sum_{v\in\mathcal{V}}q(v)\frac{\eta e}{2}\mathbb{E}\left[\sum_{t=1}^{T}\mathbb{E}_{t}\left[\sum_{i,j=1}^{k}\xbar_{t}(i,v)B_{t}(i,v)\min_{w\in\cN(v)}\left\{ G_{t}(i,w)\right\} B_{t}(j,v)\min_{w\in\cN(v)}\left\{ G_{t}(j,w)\right\} \right]\right]\;,
\end{align*}
which is rewritten as
\begin{multline*}
\sum_{v\in\mathcal{V}}q(v)\frac{\eta e}{2}\mathbb{E}
\lsb{
\sum_{t=1}^{T}\mathbb{E}_{t}\left[\sum_{i,j=1}^{k}\xbar_{t}(i,v)B_{t}(i,v)\min_{w\in\cN(v)}\left\{ G_{t}(i,w)\right\} B_{t}(j,v)\min_{w\in\cN(v)}\left\{ G_{t}(j,w)\right\} \right]
}
\\
\begin{aligned}
&=
    \sum_{v\in\mathcal{V}}q(v)\frac{\eta e}{2}\mathbb{E}\left[\sum_{t=1}^{T}\mathbb{E}_{t}\left[\sum_{i=1}^{k}\sum_{j=1}^{k}\xbar_{t}(i,v)B_{t}(i,v)\tilde{G_{t}}(i,v)B_{t}(j,v)\tilde{G}_{t}(j,v)\right]\right]\\
&=
    \sum_{v\in\mathcal{V}}q(v)\frac{\eta e}{2}\mathbb{E}\left[\sum_{t=1}^{T}\sum_{i=1}^{k}\sum_{j=1}^{k}\xbar_{t}(i,v)\mathbb{E}_{t}\left[B_{t}(i,v)B_{t}(j,v)\right]\mathbb{E}_{t}\left[\tilde{G_{t}}(i,v)\right]\mathbb{E}_{t}\left[\tilde{G}_{t}(j,v)\right]\right]
=:
    \left(\star\right)\;,
\end{aligned}
\end{multline*}
where in the first equality we defined $\tilde{G}_{t}(i,v) = \min_{w\in \cN(v)} \bcb{ G_t(i,w) }$ and, analogously, $\tilde{G}_{t}(j,v) = \min_{w\in \cN(v)} \bcb{ G_t(j,w) }$, while the second follows by the conditional independence of the three terms $\brb{ B_{t}(i,v),B_{t}(j,v) }$, $\tilde{G_{t}}(i,v)$, and $\tilde{G}_{t}(j,v)$ given the history up to time $t-1$.
Furthermore, making use of Lemmas \ref{lem:Min-Geoms}--\ref{lem:buildGeoms} and upper bounding, we get:
\begin{align*}
    \left(\star\right) 
& =
    \sum_{v\in\mathcal{V}}q(v)\frac{\eta e}{2}\mathbb{E}\left[\sum_{t=1}^{T}\mathbb{E}_{t}\left[\sum_{i=1}^{k}\sum_{j=1}^{k}\frac{\xbar_{t}(i,v)}{\overline{B}_{t}(i,v)}B_{t}(i,v)\frac{B_{t}(j,v)}{\overline{B}_{t}(j,v)}\right]\right]\\
 & \le
    \frac{\eta ek}{2}\mathbb{E}\left[\sum_{t=1}^{T}\sum_{i=1}^{k}\sum_{v\in\mathcal{V}}\frac{\xbar_{t}(i,v)q(v)}{\overline{B}_{t}(i,v)}\right]
 \leq
    \frac{\eta e kT}{2\left(1-e^{-1}\right)}\left(k\alpha_1+m Q\right),
\end{align*}
where the first equality uses the expected value of the geometric random variables $\tilde{G}$, the first inequality is obtained neglecting the indicator function
$B_{t}(i,v)$ and taking the conditional expectation of $B_t(j,v)$,
and the last inequality follows by a known upper bound involving independence numbers appearing, for example
in \cite{cesa2019cooperative,cesa2019delay}. 
For the sake of convenience, we add this result to Appendix~\ref{s:bounds-on-alpha}, Lemma~\ref{l:alpha-q-bound}.
We now consider the last term $(\mathrm{III})$.
Since $\ell_{t}\ge \bbE_t \Bsb{\widehat{\ell}_{t}(v)}$ by Lemma~\ref{lem:The-expectation-of},
we have
\begin{align*}
    (\mathrm{III})
& =
    \mathbb{E}\left[\sum_{t=1}^{T}\mathbb{E}_{t}\left[\left\langle \xbar_t(v)-a,\ell_{t}-\widehat{\ell}_{t}(v)\right\rangle \right]\right]\leq\mathbb{E}\left[\sum_{t=1}^{T}\mathbb{E}_{t}\left[\left\langle \xbar_t(v),\ell_{t}-\widehat{\ell}_{t}(v)\right\rangle \right]\right]\\
 & =
    \mathbb{E}\left[\sum_{t=1}^{T}\sum_{i=1}^{k}\ell_{t}(i)\xbar_{t}(i,v)\Bbrb{\prod_{w\in\cN(v)} \brb{ 1-\xbar_{t}(i,w)\,q(w) } }^{\beta}\right]\;.
\end{align*}
Multiplying $(\mathrm{III})$ by $q(v)$ and summing over the agents, we now upper bound $\ell_t(i)$ with $1$ and use the facts that $1-x\leq e^{-x}$ for all $x\in[0,1]$ and $e^{-y}\le 1/y$ for all $y>0$, to obtain
\begin{multline*}
    \sum_{v\in\cV}q(v)\mathbb{E}\left[\sum_{t=1}^{T}\sum_{i=1}^{k}\ell_{t}(i)\xbar_{t}(i,v)\left(\prod_{w\in\cN(v)}\left(1-\xbar_{t}(i,w)\,q(w)\right)\right)^{\beta}\right]
\\
\begin{aligned}
& \leq
    \mathbb{E}\left[\sum_{t=1}^{T}\sum_{i=1}^{k} \sum_{v\in\cV} \xbar_{t}(i,v) \, q(v) \left(\prod_{w\in\cN(v)}\left(1-\xbar_{t}(i,w)\,q(w)\right)\right)^{\beta}\right]
\\
& =
    \mathbb{E}\left[\sum_{t=1}^{T}\sum_{i=1}^{k} \sum_{\substack{v\in \cV \\ \xbar_{t}(i,v) \, q(v) > 0}} \xbar_{t}(i,v) \, q(v) \left(\prod_{w\in\cN(v)}\left(1-\xbar_{t}(i,w)\,q(w)\right)\right)^{\beta}\right]
\\
& \leq
    \mathbb{E}\left[\sum_{t=1}^{T}\sum_{i=1}^{k} \sum_{\substack{v\in \cV \\ \xbar_{t}(i,v) \, q(v) > 0}} \xbar_{t}(i,v) q(v) \exp\left(-\beta\sum_{w\in\cN(v)}\xbar_{t}(i,w)\,q(w)\right)\right]
\\
& \leq
    \mathbb{E}\left[\sum_{t=1}^{T}\sum_{i=1}^{k} \sum_{\substack{v\in \cV \\ \xbar_{t}(i,v) \, q(v) > 0}}
    \frac{\xbar_{t}(i,v) \, q(v)}{\beta\sum_{w\in\cN(v)}\xbar_{t}(i,w)\,q(w)}\right]
\leq
    \mathbb{E}\left[\sum_{t=1}^{T}\sum_{i=1}^{k} \frac{\alpha_1}{\beta}\right]
=
    \frac{\alpha_1 \, k \, T}{\beta}
\end{aligned}
\end{multline*}
where the last inequality follows by a known upper bound involving independence numbers appearing, for example in \cite[Lemma 10]{alon2017nonstochastic}. 
For the sake of convenience, we add this result to Appendix~\ref{s:bounds-on-alpha}, Lemma~\ref{lm:bound-alpha}.
Putting everything together and recalling that $\beta = \bigl \lfloor{1}/{(k \eta)} \bigr \rfloor \ge {1}/{(2 k \eta)},$ we can finally conclude that
\begin{align*}
    R_T
& \leq
    Q\frac{m\left(1+\log(k)\right)}{\eta}+Q\frac{\eta ekT}{2\left(1-e^{-1}\right)}\left(\frac{k}{Q}\alpha_1+m\right)+\frac{\alpha_1 \, k \, T}{\beta}\\
 & \leq
    Q\frac{m\left(1+\log(k)\right)}{\eta}+Q\frac{\eta ekT}{2\left(1-e^{-1}\right)}\left(\frac{k}{Q}\alpha_1+m\right)+2 \eta \alpha_1 k^2 T\\
 & =
    Q\frac{m\left(1+\log(k)\right)}{\eta}+\eta Q k T\left(\frac{e}{2\left(1-e^{-1}\right)}\left(\frac{k}{Q}\alpha_1+m\right)+2 \alpha_1 \frac{k}{Q}\right)\\
& \le
    Q\frac{m\left(1+\log(k)\right)}{\eta}+ 5 \eta Q k T \left(\frac{k}{Q}\alpha_1+m\right)
\\
& 
\le
    2Q\sqrt{ 15 mkT\log(k)\left(\frac{k}{Q}\alpha_1+m\right)} \;.
\end{align*}
\end{proof}

We conclude this section by discussing the computational complexity of our \coopftpl{} algorithm.
The next result shows that the total number of elementary operations performed by \coopftpl{} over $T$ time-steps scales with $T^{3/2}$ in the worst-case.
To the best of our knowledge, no known algorithm attains a lower worst-case computational complexity.

\begin{proposition}
\label{p:comput-complexity}
If the optimization problem~\eqref{eq:efficientOPT} can be solved with at most $c \in \bbN$ elementary operations, the worst-case computational complexity $\gamma_{\text{\coopftpl{}}}$  of each agent $v\in \cV$ running our \coopftpl{} algorithm with the optimal tuning \eqref{e:tuning} for $T$ rounds is
\[
    \gamma_{\text{\coopftpl{}}}
=
    \cO \lrb{ T^{3/2}  c \sqrt{\frac{ \alpha_1/Q + 1 }{  m }} }\;.
\]
\end{proposition}
\begin{proof}
The result follows immediately by noting that the number of elementary operations performed by each agent $v$ at each time step $t$ is at most
\[
    c(\beta + 1)
\le
    c \lrb{ \frac{1}{k\eta} + 1}
=
    c \lrb{ \frac{1}{k} \sqrt{\frac{5kT (k\alpha_1/Q + m) }{ 2 m \log k}} + 1 }
=
    c \lrb{ \sqrt{\frac{5T (\alpha_1/Q + m/k) }{ 2 m \log k}} + 1 }\;.
\]
\end{proof}

\section{Lower bound}
\label{sec:LB}

In this section we show that no cooperative semi-bandit algorithm can beat the $Q \sqrt{mkT \alpha_1/Q}$ rate.
The key idea for constructing the lower bound is simple:
if the activation probabilities $q(v)$ are non-zero only for agents $v$ belonging to an independent set with cardinality $\alpha_1$, then the problem is reduced to $\alpha_1$ independent instances of single-agent semi-bandits, whose minimax rate is known.

\begin{theorem}
\label{t:lower-bound}
For each communication network $\cG$ with independence number $\alpha_1$ there exist cooperative semi-bandit instances for which the regret of any learning algorithm satisfies \[
    R_T = \Omega \brb{Q \sqrt{mkT \alpha_1/Q}}\;.
\]
\end{theorem} 

\begin{proof}
Let $\cW = \{w_{1}, \ldots, w_{\alpha_1} \} \s \cV$ be an independent set with cardinality $\alpha_1$.
Furthermore, let $q \in (0,1]$ be a positive probability and for all agents $v \in \cV$, let 
\[
    q(v) = q \bbI \{ v \in \cW \}\;.
\]
In words, only agents belonging to an independent set with largest cardinality are activated (with positive probability), and all with the same probability.
Thus, only agents in $\cW$ contribute to the expected regret and their total mass $Q=\sum_{v\in \cV} q(v)$ is equal to $\alpha_1 q$.
Moreover, note that being non-adjacent, agents in $\cW$ never exchange any information.
Each agent $w\in \cW$ is therefore running an independent single-agent online linear optimization problem with semi-bandit feedback for an average of $q T$ rounds.
Since for single-agent semi-bandits, the worst-case lower bound on the regret after $T'$ time steps is known to be $\Omega \brb{ \sqrt{mkT'} }$ (see, e.g., \cite{audibert2014regret,lattimore2018toprank}) and the cardinality of $\cW$ is $\alpha_1$, the regret of any cooperative semi-bandit algorithm run on this instance satisfies 
\[
    R_{T}
    =\Omega\brb{\alpha_1\sqrt{mk\, qT}}
    =\Omega\brb{\alpha_1 q \sqrt{mkT/q}}
    =\Omega\brb{Q\sqrt{mkT\alpha_1/Q}}\;,
\]
where we used $Q = \alpha_1 q$. 
This concludes the proof.
\end{proof} 
In the previous section we showed that the expected regret of our \coopftpl{} algorithm can always be upper bounded by $Q\sqrt{mkT\log(k)({k\alpha_1}/{Q}+m)}$ (ignoring constants).
Thus, Theorem~\ref{t:lower-bound} shows that, up to the additive $m$ term inside the rightmost bracket, the regret of \coopftpl{} is at most $\sqrt{k\log k}$-away from the minimax optimal rate.

\section{Conclusions and open problems}

Motivated by spatially distributed large-scale learning systems, we introduced a new cooperative setting for adversarial semi-bandits in which only some of the agents are active at any given time step.
We designed and analyzed an efficient algorithm that we called \coopftpl{} for which we proved near-optimal regret guarantees with state-of-the-art computational complexity costs.
Our analysis relies on the fact that agents are aware of their activation probabilities, and they have some prior knowledge about the connectivity of the graph.
Two interesting new lines of research are investigating if either of these assumptions could be lifted while retaining low regret and good computational complexity.
In particular, removing the need for prior knowledge of the independence number would represent a significant theoretical and practical improvement, given that computing $\alpha_1$ is NP-hard in the worst-case.
It is however unclear if existing techniques that address this problem in similar settings (e.g., \cite{cesa2019delay}) would yield any results in our general case.
We believe that entirely new ideas will be required to deal with this issue.
We leave these intriguing problems open for future work.

\section*{Acknowledgements}
This project has received funding from the French ``Investing for the Future – PIA3'' program under the Grant agreement ANITI ANR-19-PI3A-0004.\footnote{\url{https://aniti.univ-toulouse.fr/}}

\bibliography{biblio}

\appendix

\section{Legendre functions and \fen{} conjugates}
\label{s:convex-recap}

In this section, we briefly recall a few known definitions and facts in convex analysis.

\begin{definition}[Interior, boundary, and convex hull]
For any subset $E$ of $\R^k$, we denote its \emph{topological interior} by $\Int(E)$, its \emph{boundary} by $\partial E$, and its \emph{convex hull} by $\conv(E)$.
\end{definition}

\begin{definition}[Effective domain]
The \emph{effective domain} of a convex function $F\colon \bbR^k \to \Rbar$ is
\begin{equation}
    \label{e:domain}
    \dom(F)
:=
    \bcb{ x \in \bbR^k : F(x) < + \iop } \;.
\end{equation}
\end{definition}

With a slight abuse of notation, we will denote with the same symbol $f$ a convex function $f\colon \to \Rbar$ and its restriction $\widetilde{f} \colon \dom(f) \to \bbR$ to its effective domain.

\begin{definition}[\leg{} function]
A convex function $F\colon\bbR^k \to \Rbar$ is \emph{\leg} if
\begin{enumerate}[topsep = 0pt, parsep = 0pt, itemsep = 0pt]
\item $\Int\brb{ \dom(F) }$ is non-empty;
\item $F$ is differentiable and strictly convex on $\Int\brb{ \dom(F) }$;
\item for all $x_0 \in \partial \lsb{ \Int\brb{ \dom(F) } }$, if $x \in \Int \brb{ \dom(F) }, x \to x_0$, then $\lno{ \nabla F (x) }_2 \to +\infty$.
\end{enumerate}
\end{definition}

\begin{definition}[\fen{} conjugate]
Let $F\colon \bbR^k \to \Rbar$ be a convex function.
The \emph{\fen{} conjugate} $F^*$ of $F$ is defined as the function
\begin{align*}
    F^* \colon \bbR^k & \to \Rbar \\
    z & \mapsto F^*(z) := \sup_{x \in \bbR^k} \brb{ \lan{ x, z } - F(x) } \;.
\end{align*}
\end{definition}

\begin{definition}[\breg]
\label{d:breg}
Let $F \colon \bbR^k \to \Rbar$ a convex function with non-empty $\Int \brb{\dom(F)}$ that is differentiable on $\Int \brb{\dom(F)}$.
The \emph{\breg{}} induced by $F$ is
\begin{align*}
    \cB_ F \colon \bbR^k \times \Int\brb{\dom(F)} & \to \Rbar \\
    (x,y) & \mapsto \cB_F (x,y) :=  F(x) -  F(y) - \ban{ \nabla  F(y), x-y } \;.
\end{align*}
\end{definition}

The following results are taken from \cite[Theorem~26.6 and Corollary~26.8]{lattimore2020bandit}.

\begin{theorem}
\label{t:fenchel-stuff}
Let $F\colon \R^k \to \Rbar$ be a \leg{} function. Then:
\begin{enumerate}[topsep = 0pt, parsep = 0pt, itemsep = 0pt]
\item the \fen{} conjugate $F^{*}$ of $F$ is \leg{};
\item $\nabla F \colon \Int\brb{\dom(F)} \to \Int\brb{ \dom(F^{*}) }$ is bijective with inverse $(\nabla F)^{-1}=\nabla F^{*}$;
\item $\cB_F(x, y) = \cB_{F^{*}} \brb{\nabla F(y), \nabla F(x)}$, for all $x, y \in \Int\brb{ \dom(f) }$.
\end{enumerate}
\end{theorem}

\begin{corollary}
If $F\colon \bbR^k \to \Rbar$ is a \leg{} function and $x \in \argmin_{x \in \dom (F)} F(x)$, then $x \in \Int\brb{\dom(F)}$.
\end{corollary}

\section{Online Stochastic Mirror Descent (OSMD)}
\label{s:osmd}

In this section, we briefly recall the standard Online Stochastic Mirror Descent algorithm (OSMD) (Algorithm~\ref{algo:OSMD}) and its analysis.

For an overview on some basic convex analysis definitions and results, we refer the reviewer to the previous Appendix~\ref{s:convex-recap}.
For a convex function $F\colon \bbR^k \to \Rbar$ that is differentiable on the non-empty interior $\Int\brb{\dom(F)}\neq \varnothing$ of its effective domain $\dom(F)$, we denote by $\cB_F$ the \breg{} induced by $F$ (Definition~\ref{d:breg}).
Following the existing convention, we refer to the input function $F$ of OSMD as a \emph{potential}.
Furthermore, given a measure $P$ on a subset of $\R^k$, we say that a vector $x \in \R^k$ is the \emph{mean of the measure} $P$ if $x$ is the component-wise expectation of a $\R^k$-valued random variable with distribution $P$.
For any time step $t \in \{1,2,\ld\}$, we denote by $\bbE_t$ the expectation conditioned to the history up to and including round $t-1$.

\begin{algorithm2e}
    \LinesNumbered
    \SetAlgoNoLine
    \SetAlgoNoEnd
    \DontPrintSemicolon
    \SetKwInput{kwInit}{Initialization}
    \KwIn{\leg{} potential $F\colon \bbR^k \to \Rbar$, compact action set $\calA \s \bbR^k$ with $\Int\brb{ \dom(F) } \cap \conv(\calA) \neq \varnothing$, learning rate $\eta>0$}
    \kwInit{$\xbar_{1}=\argmin_{x \in \dom(F)\cap \conv(\calA)} F(x)$} 
\For{$t=1,2,\ld$}
    {
        choose a measure $P_{t}$ on $\calA$ with mean $\xbar_{t}$\;
        make a prediction $x_{t}$ drawn from $\calA$ according to $P_{t}$ and suffer the loss $\lan{ x_{t}, \ell_{t} }$\;
        compute an estimate $\lhat_{t}$ of the loss vector $\ell_{t}$\label{st:estim}\;
        make the update: $
            \xbar_{t+1}
        =
            \argmin_{x \in \dom(F)\cap \conv(\calA)} \eta \ban{ x, \lhat_{t} }+\cB_{F} (x, \xbar_{t})$\;
    }

\caption{\label{algo:OSMD}Online Stochastic Mirror Descent (OSMD)}
\end{algorithm2e}

It is known that since $\conv(\calA)$ is convex and compact, $\Int\brb{ \dom(F) } \cap \conv(\calA) \neq \varnothing$, and $F$ is \leg{}, then, all the $\argmin$'s exist in Algorithm~\ref{algo:OSMD} and $\xbar_t \in \Int\brb{\dom(F)} \cap \conv(\calA)$ for all $t\in \{1,2,\ld\}$ (see, e.g., \cite[Exercise~28.2]{lattimore2020bandit}).

The following result is taken from \cite[Theorem 28.10]{lattimore2020bandit} and gives an upper bound on the regret of OSMD.

\begin{theorem}\label{thm:regretOSMD}
Suppose that OSMD (Algorithm~\ref{algo:OSMD}) is run with input $F,\calA,\eta$.
If the estimates $\lhat_t$ computed at line~\ref{st:estim} satisfy $\bbE_t \bsb{ \lhat_t  } = \ell_t$ for all $t\in\{1,2,\ld\}$, then, for all $x\in \conv(\calA)$,
\[
    \bbE \Bbsb{\sum_{t=1}^T  \lan{ \xbar_t, \ell_t }  - \sum_{t=1}^T \lan{ x, \ell_t }}
\le 
    \bbE\lsb{
        \frac{F(x)-F(\xbar_{1})}{\eta}
        + \sum_{t=1}^{T}\ban{ \xbar_{t} - \xbar_{t+1}, \lhat_{t} }
        - \frac{1}{\eta} \sum_{t=1}^{T} \cB_F (\xbar_{t+1}, \xbar_{t}) 
    }\;.
\]
Furthermore, letting
\[
    \xt_{t+1}
=
    \argmin_{x \in \dom(F)} \eta \ban{ x, \widehat{\ell}_{t} }+\cB_{F} (x, \xbar_{t} )
\] 
and assuming that $-\eta \lhat_{t} + \nabla F(x) \in \nabla F\brb{\dom(F)}$ for all $x \in \conv(\calA)$ almost surely, then
\[
   \sup_{x \in \conv(\calA)} \bbE \Bbsb{\sum_{t=1}^T  \lan{ \xbar_t, \ell_t }  - \sum_{t=1}^T \lan{ x, \ell_t }}
\le 
    \frac{\diam_{F}\brb{\coa }}{\eta}
    +\frac{1}{\eta} \sum_{t=1}^{T} \bbE \bsb{ \cB_F (\xbar_{t}, \xt_{t+1}) } \;,
\]
where 
$
    \diam_F\brb{\coa}
:= 
    \sup_{x,y \in \coa} \brb{ F(x) - F(y) }
$ 
is the \emph{diameter} of $\coa$ with respect to $F$.
\end{theorem}

\section{Proofs of lemmas on geometric distributions}
\label{s:geometric-appe}

In this section we provide all missing proofs on geometric distributions that we stated in Section~\ref{s:upper-bound}.

\begin{proof}\textbf{of Lemma~\ref{lem:Min-Geoms}}
For all $j\in \{1,\ld,m\}$, the cumulative distribution function (c.d.f.) of $Y_{j}$ is  given, for all $n\in \bbN$, by
\[
    \mathbb{P}\left[Y_{j}\leq n\right]
=
    p_{j}\sum_{i=1}^{n}\left(1-p_{j}\right)^{i-1}=1-\left(1-p_{j}\right)^{n} \;.
\]
The c.d.f. of $Z$ is given, for all $n\in \bbN$, by
\begin{align*}
    \mathbb{P}\left[Z\leq n\right] 
& =
    \mathbb{P}\left[\min_{j\in \{1,\ld,m\}} Y_j\leq n\right]=1-\prod_{j=1}^m\mathbb{P}\left[Y_{j}>n\right]
=
    1-\prod_{j=1}^m\left(1-\mathbb{P}\left[Y_{j}\leq n\right]\right)\\
 & =
    1-\prod_{j=1}^m\left(1-\left(1-\left(1-p_{j}\right)^{n}\right)\right)
 =
    1-\left(\prod_{j=1}^m\left(1-p_{j}\right)\right)^{n}\\
 & =
    1-\left(1-\left[1-\prod_{j=1}^m\left(1-p_{j}\right)\right]\right)^{n} \;,
\end{align*}
and this is the c.d.f. of a geometric random variable with parameter $(1-\prod_{j=1}^m\left(1-p_{j}\right)$ \;.
\end{proof}

\begin{proof}\textbf{of Lemma~\ref{lem:Min-Geom-Beta}}
By elementary calculations,
\begin{align*}
\mathbb{E}\left[\min\left\{ G,\beta\right\} \right] & =\sum_{n=1}^{\infty}n\left(1-q\right)^{n-1}q-\sum_{n=\beta}^{\infty}\left(n-\beta\right)\left(1-q\right)^{n-1}q\\
 & =\sum_{n=1}^{\infty}n\left(1-q\right)^{n-1}q-\left(1-q\right)^{\beta}\sum_{n=\beta}^{\infty}\left(n-\beta\right)\left(1-q\right)^{n-\beta-1}q\\
 & =\left(1-\left(1-q\right)^{\beta}\right)\sum_{n=1}^{\infty}n\left(1-q\right)^{n-1}q=\frac{\left(1-\left(1-q\right)^{\beta}\right)}{q} \;.
\end{align*}
\end{proof}

\begin{proof}\textbf{of Lemma~\ref{lem:buildGeoms}}
The proof follows immediately from the fact that $\bcb{ X_s (v) \, Y_s(v) }_{ s \in \cI, v \in \cV }$ is a collection of independent Bernoulli random variables with expectation $\bbE \bsb{ X_s (v) \, Y_s(v) } = p_1(v) \, p_2(v)$ for any $s\in \bbN$ and all $v\in \cV$.
\end{proof}

\section{Proof of Theorem~\ref{thm:main}}\label{app:proof}

\label{s:proof-main-thm}

In this section, we present a complete proof of Theorem~\ref{thm:main}.

\begin{proof}\textbf{of Theorem~\ref{thm:main}}
For the sake of convenience, we define the expected individual regret of an agent $v\in \cV$ in the network with respect to a fixed action $a \in \calA$ by
\[
    R_T(a,v)
:=
    \mathbb{E}\left[\sum_{t=1}^{T}\left\langle x_{t}(v),\ell_{t}\right\rangle-\sum_{t=1}^{T}\left\langle a,\ell_{t}\right\rangle \right]\;,
\]
where the expectation is taken with respect to the internal randomization of the agent, but not its activation probability $q(v)$. With this definition the total regret on the network in Eq.~\eqref{eq:totalExpRegret} can be decomposed as 
\begin{align}
    R_T \nonumber
&=
    \max_{a \in \calA}\bbE\lsb{\sum_{t=1}^{T}\sum_{v\in \cA_t}\Brb{ \ban{ x_{t}(v),\ell_{t} } -\left\langle a,\ell_{t}\right\rangle}}
=
    \max_{a \in \calA}\bbE\lsb{\sum_{t=1}^{T}\bbE_t\lsb{\sum_{v\in \cA_t}\Brb{ \ban{ x_{t}(v),\ell_{t} } -\left\langle a,\ell_{t}\right\rangle}}}\\
&=    
    \max_{a \in \calA}\bbE\lsb{\sum_{t=1}^{T}\sum_{v\in\cV} q(v) \bbE_t\Bsb{\ban{ x_{t}(v),\ell_{t} }-\left\langle a,\ell_{t}\right\rangle}}
=    
    \max_{a \in \calA}\sum_{v\in\cV} q(v) R_T(a,v)\;.\label{eq:decompositionRegret}
\end{align}
The proof then proceeds by isolating the bias in the loss estimators. 
For each $a\in \calA$ we get
\begin{multline*}
R_{T}(a,v) 
\\
\begin{aligned}
& =
    \mathbb{E}\left[\sum_{t=1}^{T}\left\langle x_{t}(v)-a,\ell_{t}\right\rangle \right]=\mathbb{E}\left[\mathbb{E}_{t}\left[\sum_{t=1}^{T}\left\langle x_{t}(v)-a,\ell_{t}\right\rangle \right]\right]=\mathbb{E}\left[\sum_{t=1}^{T}\left\langle \xbar_t(v)-a,\ell_{t}\right\rangle \right]\\
 & =
    \mathbb{E}\left[\sum_{t=1}^{T}\left\langle \xbar_t(v)-a,\widehat{\ell}_{t}(v)\right\rangle \right]+\mathbb{E}\left[\sum_{t=1}^{T}\left\langle \xbar_t(v)-a,\ell_{t}-\widehat{\ell}_{t}(v)\right\rangle \right]\\
 & \leq
    \underbrace{\frac{F(\xbar_{1}(v))-F(a)}{\eta}}_{(\mathrm{I})}
    + \underbrace{\mathbb{E}\left[\frac{1}{\eta}\sum_{t=1}^{T}\cB_{F}\left(\xbar_t(v),\xbar_{t+1}(v)\right)\right]}_{(\mathrm{II})}
    +\underbrace{\mathbb{E}\left[\sum_{t=1}^{T}\left\langle \xbar_t(v)-a,\ell_{t}-\widehat{\ell}_{t}(v)\right\rangle \right]}_{(\mathrm{III})}
\end{aligned}
\end{multline*}
where the inequality follows by the standard analysis of OMD. %
We bound the three terms separately. For the first term $(\mathrm{I})$, we have
\begin{align}
F(a) 
& =
    \sup_{x\in\mathbb{R}^{k}} \brb{ \left\langle a,x\right\rangle -F^{*}(x) }
=
    \sup_{x\in\mathbb{R}^{k}}\left(\left\langle a,x\right\rangle -\mathbb{E}\left[\max_{a'\in\mathcal{A}}\left\langle a',x+Z\right\rangle \right]\right)\nonumber \\
 & \geq
    -\mathbb{E}\left[\max_{a'\in\mathcal{A}}\left\langle a',Z\right\rangle \right]\geq-m\mathbb{E}\left[\left\Vert Z\right\Vert _{\infty}\right]=-m\sum_{i=1}^{k}\frac{1}{i}\geq-m\left(1+\log(k)\right),\label{eq:combinatorialDiam}
\end{align}
where the first inequality follows by choosing $x=0$, the second
follows from H\"older's inequality and $\left\Vert a\right\Vert _{1}\leq m$
for any $a\in\mathcal{A}$, and the last equality is Exercise 30.4
in \cite{lattimore2020bandit}. It follows that
\[
F \brb{ \xbar_{1}(v) } -F(a)\leq m \brb{ 1+\log(k) }\;,
\]
where we use the fact that $F(a)\leq0$ for all $a\in\mathcal{A}$
and this follows from the first line of Eq. (\ref{eq:combinatorialDiam})
by the convexity of the maximum of a finite number of linear functions, using Jensen's inequality and the
fact that the random variable $Z$ is centered.
Thus
\[
    (\mathrm{I})
\le
    \frac{m (1+\log k)}{\eta} \;.
\]
We now study the second term $(\mathrm{II}).$
We have 

\begin{align}
    \cB_{F}\left(\xbar_{t}(v),\xbar_{t+1}(v)\right) 
& =
    \cB_{F^{*}}\left(\nabla F\left(\xbar_{t+1}(v)\right),\nabla F\left(\xbar_t(v)\right)\right)\nonumber \\
 & =
    \cB_{F^{*}}\left(-\eta\widehat{L}_{t-1}(v)-\eta\widehat{\ell}_{t}(v),-\eta\widehat{L}_{t-1}(v)\right)\nonumber \\
 & =
    \frac{\eta^{2}}{2}\left\Vert \widehat{\ell}_{t}(v)\right\Vert _{\nabla^{2}F^{*}(\xi(v))}^{2}\;,\label{eq:divereCOmb}
\end{align}
where the first equality is a standard property of Bregmann divergence,
the second follows from the definitions of the updates and the last
by Taylor's theorem, where 
$\xi(v)=-\eta \widehat{L}_{t-1}(v)-\alpha \eta \widehat{\ell}_{t}(v)$, for some $\alpha \in [0,1]$.
To calculate the Hessian we use a change of variable to avoid applying
the gradient to the (possibly) non-differentiable $\argmax$ and we get:%
\[
\begin{aligned}
\nabla^{2} F^{*}(x) 
&=
    \nabla\left(\nabla F^{*}(x)\right)=\nabla \mathbb{E}[h(x+Z)]=\nabla \int_{\mathbb{R}^{k}} h(x+z) \zeta(z) d z \\
&=
    \nabla \int_{\mathbb{R}^{k}} h(u) \zeta(u-x) d u=\int_{\mathbb{R}^{k}} h(u)(\nabla \zeta(u-x))^{\top} d u \\
&=
    \int_{\mathbb{R}^{k}} h(u) \sign(u-x)^{\top} \zeta(u-x) d u=\int_{\mathbb{R}^{k}} h(x+z) \sign(z)^{\top} \zeta(z) d z
\end{aligned}
\]
Using the definition of $\xi(v)$ and the fact that $h(x)$ is nonnegative,
\[
\begin{aligned}
    \nabla^{2} F^{*}(\xi(v))_{i j} 
&=
    \int_{\mathbb{R}^{k}} h(\xi(v)+z)_{i} \sign(z)_{j} \zeta(z) d z \\
& \leq 
    \int_{\mathbb{R}^{k}} h(\xi(v)+z)_{i} \zeta(z) d z \\
&=
    \int_{\mathbb{R}^{k}} h\left(z-\eta \widehat{L}_{t-1}(v)-\alpha \eta \widehat{\ell}_{t}(v)\right)_{i} \zeta(z) d z \\
&=
    \int_{\mathbb{R}^{k}} h\left(u-\eta \widehat{L}_{t-1}(v)\right)_{i} \zeta\left(u+\alpha \eta \widehat{\ell}_{t}(v)\right) d u \\
&  \leq
    \exp \left(\left\|\alpha \eta \widehat{\ell}_{t}(v)\right\|_{1}\right) \int_{\mathbb{R}^{k}} h\left(u-\eta \widehat{L}_{t-1}(v)\right)_{i} \zeta(u) d u \\
& \leq
    \exp \lrb{ \alpha \eta \sum_{i=1}^k B_t(i,v) \beta  } \xbar_{t}(i,v)   \\
& \leq
    \exp \lrb{ \alpha \eta k \beta  } \xbar_{t}(i,v)   \\
& \leq
    e \, \xbar_{t}(i,v)
\end{aligned}
\]
where the last inequality follows by $\alpha \le 1$ and $\beta \le 1 /(\eta k)$. 
Plugging this estimate in Eq. (\ref{eq:divereCOmb}) yields
\begin{align*}
\frac{\eta^{2}}{2}\left\Vert \widehat{\ell}_{t}(v)\right\Vert _{\nabla^{2}F^{*}(\xi(v))}^{2} 
& =
    \frac{\eta^{2}}{2}\sum_{i=1}^{k}\sum_{j=1}^{k}\nabla^{2}F^{*}\brb{ \xi(v) }_{i,j}\widehat{\ell}_{t}(i,v)\widehat{\ell}_{t}(j,v)\\
 & \leq
    \frac{\eta^{2} e}{2}\sum_{i=1}^{k}\sum_{j=1}^{k}\xbar_{t}(i,v)\widehat{\ell}_{t}(i,v)\widehat{\ell}_{t}(j,v)\\
 & \leq
    \frac{\eta^{2}e}{2}\sum_{i=1}^{k}\sum_{j=1}^{k}\xbar_{t}(i,v)B_{t}(i,v)\min_{w\in\cN(v)}\left\{ G_{t}(i,w)\right\} B_{t}(j,v)\min_{w\in\cN(v)}\left\{ G_{t}(j,w)\right\} ,
\end{align*}
where the last inequality follows by neglecting the truncation with
$\beta$. Hence multiplying $(\mathrm{II})$ by $q(v)$ and summing over $v\in \cV$ yields
\begin{align*}
& 
    \sum_{v\in\mathcal{V}}q(v)\mathbb{E}\left[\frac{\eta}{2}\sum_{t=1}^{T}\left\Vert \widehat{\ell}_{t}(v)\right\Vert _{\nabla^{2}F^{*}(\xi(v))}^{2}\right]
=
    \sum_{v\in \cV}q(v)\frac{\eta }{2}\mathbb{E}\left[\sum_{t=1}^{T}\mathbb{E}_{t}\left[\left\Vert \widehat{\ell}_{t}(v)\right\Vert _{\nabla^{2}F^{*}(\xi(v))}^{2}\right]\right]
\\
& \hspace{2.94753pt} \leq
    \sum_{v\in\mathcal{V}}q(v)\frac{\eta e}{2}\mathbb{E}\left[\sum_{t=1}^{T}\mathbb{E}_{t}\left[\sum_{i,j=1}^{k}\xbar_{t}(i,v)B_{t}(i,v)\min_{w\in\cN(v)}\left\{ G_{t}(i,w)\right\} B_{t}(j,v)\min_{w\in\cN(v)}\left\{ G_{t}(j,w)\right\} \right]\right]\;,
\end{align*}
which is rewritten as
\begin{multline*}
\sum_{v\in\mathcal{V}}q(v)\frac{\eta e}{2}\mathbb{E}
\lsb{
\sum_{t=1}^{T}\mathbb{E}_{t}\left[\sum_{i,j=1}^{k}\xbar_{t}(i,v)B_{t}(i,v)\min_{w\in\cN(v)}\left\{ G_{t}(i,w)\right\} B_{t}(j,v)\min_{w\in\cN(v)}\left\{ G_{t}(j,w)\right\} \right]
}
\\
\begin{aligned}
&=
    \sum_{v\in\mathcal{V}}q(v)\frac{\eta e}{2}\mathbb{E}\left[\sum_{t=1}^{T}\mathbb{E}_{t}\left[\sum_{i=1}^{k}\sum_{j=1}^{k}\xbar_{t}(i,v)B_{t}(i,v)\tilde{G_{t}}(i,v)B_{t}(j,v)\tilde{G}_{t}(j,v)\right]\right]\\
&=
    \sum_{v\in\mathcal{V}}q(v)\frac{\eta e}{2}\mathbb{E}\left[\sum_{t=1}^{T}\sum_{i=1}^{k}\sum_{j=1}^{k}\xbar_{t}(i,v)\mathbb{E}_{t}\left[B_{t}(i,v)B_{t}(j,v)\right]\mathbb{E}_{t}\left[\tilde{G_{t}}(i,v)\right]\mathbb{E}_{t}\left[\tilde{G}_{t}(j,v)\right]\right]
=:
    \left(\star\right)\;,
\end{aligned}
\end{multline*}
where in the first equality we defined $\tilde{G}_{t}(i,v) = \min_{w\in \cN(v)} \bcb{ G_t(i,w) }$ and, analogously, $\tilde{G}_{t}(j,v) = \min_{w\in \cN(v)} \bcb{ G_t(j,w) }$, while the second follows by the conditional independence of the three terms $\brb{ B_{t}(i,v),B_{t}(j,v) }$, $\tilde{G_{t}}(i,v)$, and $\tilde{G}_{t}(j,v)$ given the history up to time $t-1$.
Furthermore, making use of Lemmas \ref{lem:Min-Geoms}--\ref{lem:buildGeoms}, we get  
\begin{align*}
    \left(\star\right) 
& =
    \sum_{v\in\mathcal{V}}q(v)\frac{\eta e}{2}\mathbb{E}\left[\sum_{t=1}^{T}\mathbb{E}_{t}\left[\sum_{i=1}^{k}\sum_{j=1}^{k}\frac{\xbar_{t}(i,v)}{\overline{B}_{t}(i,v)}B_{t}(i,v)\frac{B_{t}(j,v)}{\overline{B}_{t}(j,v)}\right]\right]\\
 & \leq
    \sum_{v\in\mathcal{V}}q(v)\frac{\eta e}{2}\mathbb{E}\left[\sum_{t=1}^{T}\mathbb{E}_{t}\left[\sum_{i=1}^{k}\sum_{j=1}^{k}\frac{\xbar_{t}(i,v)}{\overline{B}_{t}(i,v)}\frac{B_{t}(j,v)}{\overline{B}_{t}(j,v)}\right]\right]\\
 & =
    \sum_{v\in\mathcal{V}}q(v)\frac{\eta e}{2}\mathbb{E}\left[\sum_{t=1}^{T}\sum_{i=1}^{k}\sum_{j=1}^{k}\frac{\xbar_{t}(i,v)}{\overline{B}_{t}(i,v)}\frac{\cancel{\overline{B}_{t}(j,v)}}{\cancel{\overline{B}_{t}(j,v)}}\right]\\
 & =
    \frac{\eta ek}{2}\mathbb{E}\left[\sum_{t=1}^{T}\sum_{i=1}^{k}\sum_{v\in\mathcal{V}}\frac{\xbar_{t}(i,v)q(v)}{\overline{B}_{t}(i,v)}\right]\\
 & \leq
    \frac{\eta ek}{2}\mathbb{E}\left[\sum_{t=1}^{T}\sum_{i=1}^{k}\left(\frac{1}{1-e^{-1}}\left(\alpha_1+\sum_{v\in\mathcal{V}}\xbar_{t}(i,v)q(v)\right)\right)\right]\\
 & =
    \frac{\eta ek}{2}\mathbb{E}\left[\sum_{t=1}^{T}\left(\frac{1}{1-e^{-1}}\left(k\alpha_1+\sum_{v\in\mathcal{V}}\sum_{i=1}^{k}\xbar_{t}(i,v)q(v)\right)\right)\right]\\
 & \leq
    \frac{\eta e kT}{2\left(1-e^{-1}\right)}\left(k\alpha_1+m Q\right)\;,
\end{align*}
where the first equality uses the expected value of the geometric
random variables $\tilde{G}$%
, the first inequality is obtained neglecting the indicator function
$B_{t}(i,v)$, the following equality uses the expected value of the
geometric random variables $B_{t}$,%
{} the second inequality follows by Lemma~\ref{l:alpha-q-bound}.
We now consider the last term $(\mathrm{III})$.
Since $\ell_{t}\geq \bbE[ \widehat{\ell}_{t}(v)]$, from Lemma \ref{lem:The-expectation-of},
we have
\begin{align*}
    (\mathrm{III})
& =
    \mathbb{E}\left[\sum_{t=1}^{T}\mathbb{E}_{t}\left[\left\langle \xbar_t(v)-a,\ell_{t}-\widehat{\ell}_{t}(v)\right\rangle \right]\right]\leq\mathbb{E}\left[\sum_{t=1}^{T}\mathbb{E}_{t}\left[\left\langle \xbar_t(v),\ell_{t}-\widehat{\ell}_{t}(v)\right\rangle \right]\right]\\
 & =
    \mathbb{E}\left[\sum_{t=1}^{T}\sum_{i=1}^{k}\ell_{t}(i)\xbar_{t}(i,v)\left(\prod_{w\in\cN(v)}\left(1-\xbar_{t}(i,w)\,q(w)\right)\right)^{\beta}\right]\;.
\end{align*}
Multiplying $(\mathrm{III})$ by $q(v)$ and summing over the agents, we can now upper bound $\ell_t(i)$ with $1$, then we use facts that $1-x\leq e^{-x}$ for $x\in[0,1]$ and that $e^{-y}\le 1/y$ for all $y>0$, to obtain
\begin{multline*}
    \sum_{v\in\cV}q(v)\mathbb{E}\left[\sum_{t=1}^{T}\sum_{i=1}^{k}\ell_{t}(i)\xbar_{t}(i,v)\left(\prod_{w\in\cN(v)}\left(1-\xbar_{t}(i,w)\,q(w)\right)\right)^{\beta}\right]
\\
\begin{aligned}
& \leq
    \mathbb{E}\left[\sum_{t=1}^{T}\sum_{i=1}^{k} \sum_{v\in\cV} \xbar_{t}(i,v) \, q(v) \left(\prod_{w\in\cN(v)}\left(1-\xbar_{t}(i,w)\,q(w)\right)\right)^{\beta}\right]
\\
& =
    \mathbb{E}\left[\sum_{t=1}^{T}\sum_{i=1}^{k} \sum_{\substack{v\in \cV \\ \xbar_{t}(i,v) \, q(v) > 0}} \xbar_{t}(i,v) \, q(v) \left(\prod_{w\in\cN(v)}\left(1-\xbar_{t}(i,w)\,q(w)\right)\right)^{\beta}\right]
\\
& \leq
    \mathbb{E}\left[\sum_{t=1}^{T}\sum_{i=1}^{k} \sum_{\substack{v\in \cV \\ \xbar_{t}(i,v) \, q(v) > 0}} \xbar_{t}(i,v) q(v) \exp\left(-\beta\sum_{w\in\cN(v)}\xbar_{t}(i,w)\,q(w)\right)\right]
\\
& \leq
    \mathbb{E}\left[\sum_{t=1}^{T}\sum_{i=1}^{k} \sum_{\substack{v\in \cV \\ \xbar_{t}(i,v) \, q(v) > 0}}
    \frac{\xbar_{t}(i,v) \, q(v)}{\beta\sum_{w\in\cN(v)}\xbar_{t}(i,w)\,q(w)}\right]
    \\
&\leq
    \mathbb{E}\left[\sum_{t=1}^{T}\sum_{i=1}^{k} \frac{\alpha_1}{\beta}\right]
=
    \frac{\alpha_1 \, k \, T}{\beta}
\end{aligned}
\end{multline*}
where in the last inequality follows by Lemma~\ref{lm:bound-alpha}.
Putting all together and recalling that $\beta = \left\lfloor \frac{1}{k \eta} \right\rfloor \ge \frac{1}{2 k \eta},$ we can conclude that for every $a\in\calA,$ thanks to Eq.~\eqref{eq:decompositionRegret}, we have
\begin{align*}
    R_T
& \leq
    Q\frac{m\left(1+\log(k)\right)}{\eta}+Q\frac{\eta ekT}{2\left(1-e^{-1}\right)}\left(\frac{k}{Q}\alpha_1+m\right)+\frac{\alpha_1 \, k \, T}{\beta}\\
 & \leq
    Q\frac{m\left(1+\log(k)\right)}{\eta}+Q\frac{\eta ekT}{2\left(1-e^{-1}\right)}\left(\frac{k}{Q}\alpha_1+m\right)+2 \eta \alpha_1 k^2 T\\
 & =
    Q\frac{m\left(1+\log(k)\right)}{\eta}+\eta Q k T\left(\frac{e}{2\left(1-e^{-1}\right)}\left(\frac{k}{Q}\alpha_1+m\right)+2 \alpha_1 \frac{k}{Q}\right)\\
& \le
    Q\frac{m\left(1+\log(k)\right)}{\eta}+ 5 \eta Q k T \left(\frac{k}{Q}\alpha_1+m\right)\\
& \le
    2Q\sqrt{ 15 mkT\log(k)\left(\frac{k}{Q}\alpha_1+m\right)} \;.
\end{align*}
\end{proof}

\section{Bounds on independence numbers}
\label{s:bounds-on-alpha}

The two following lemmas provide upper bounds of sums of weights over nodes of a graph expressed in terms of its independence number.

\begin{lemma}
\label{lm:bound-alpha}
Let $\cG=(\cV,\cE)$ be an undirected graph with indedence number $\alpha_1$, $q(v) \ge 0$, and $Q(v) = \sum_{w\in \cN(v)}q(w) > 0$ for all $v\in \cV$. 
Then
\[
	\sum_{v\in \cV} \frac{q(v)}{Q(v)}
\le
	\alpha_1
\]
\end{lemma}
\begin{proof}
Initialize $V_1=\cV$, fix $w_1\in\argmin_{w\in V_1}Q(w)$, and denote $V_{2}=\cV\m \cN(w_{1})$. For $k\geq 2$ fix $w_{k}\in\argmin_{w\in V_{k}}Q(w)$ and shrink $V_{k+1}=V_{k}\m \cN(w_{k})$ until $V_{k+1}=\varnothing$. Since $\cG$ is undirected $w_{k}\notin\bigcup_{s=1}^{k-1}\cN(w_{s})$, therefore the number $m$ of times that an action can be picked this way is upper bounded by $\alpha_1$. Denoting $\cN'(w_k)=V_{k}\cap \cN(w_k)$ this implies
\begin{align*}
	\sum_{v\in \cV}\frac{q(v)}{Q(v)}
& =
	\sum_{k=1}^{m}\sum_{v\in \cN'(w_k)}\frac{q(v)}{Q(v)}
\le
	\sum_{k=1}^{m}\sum_{v\in \cN'(w_k)}\frac{q(v)}{Q(w_{k})}
\\ & \le
	\sum_{k=1}^{m}\frac{\sum_{v\in \cN(w_k)}q(v)}{Q(w_{k})}
=
	m
\le
	\alpha_1
\end{align*}
concluding the proof.
\end{proof}

\begin{lemma}
\label{l:alpha-q-bound}
Let $\cG = (\cV,\cE)$ be an undirected graph with independence number $\alpha_1$. For each $v \in \cV$, let $\cN(v)$ be the neighborhood of node $v$ (including $v$ itself), and $p(1,v),\dots,p(k,v) \ge 0$. 
Then, for all $i\in \{1,\dots,k\}$,
\[
    \sum_{v \in \cV} \frac{p(i,v)}{q(i,v)}
\le
    \frac{1}{1-e^{-1}}\left(\alpha_1 + \sum_{v\in \cV} p(i,v) \right) \quad
\text{where} \quad q(i,v) = 1 - \prod_{w \in \cN(v)}\bigl(1-p(i,w)\bigr)~.
\]
\end{lemma}
\begin{proof}
Fix $i \in\{1,\dots,k\}$ and set for brevity $P(i,v) = \sum_{w \in \cN(v)} p(i,w)$. We can write
\begin{align*}
    \sum_{v \in \cV} \frac{p(i,v)}{q(i,v)}
&=
    \underbrace{\sum_{v \in \cV \,:\, P(i,v) \ge 1} \frac{p(i,v)}{q(i,v)}}_{\mathrm{(I)}}
\quad + \quad
    \underbrace{\sum_{v \in \cV\,:\, P(i,v) < 1} \frac{p(i,v)}{q(i,v)}}_{\mathrm{(II)}}~,
\end{align*}
and proceed by upper bounding the two terms~(I) and~(II) separately. Let $r(v)$ be the cardinality of $\cN(v)$. We have, for any given $v \in \cV$,
\[
    \min\left\{  q(i,v) \,:\, \sum_{w \in \cN(v)} p(i,w) \ge 1 \right\}
=
    1-\left(1-\frac{1}{r(v)}\right)^{r(v)}
\ge
    1-e^{-1}~.
\]
The equality is due to the fact that the minimum is achieved when $p(i,w) = \frac{1}{r(v)}$ for all $w \in \cN(v)$, and the inequality comes from $r(v) \ge 1$ (for, $v \in \cN(v)$). Hence
\begin{align*}
    \mathrm{(I)}
\le
    \sum_{v \in \cV \,:\, P(i,v) \ge 1} \frac{p(i,v)}{1-e^{-1}}
\le
    \sum_{v \in \cV} \frac{p(i,v)}{1-e^{-1}}~.
\end{align*}
As for~(II), using the inequality $1-x \leq e^{-x}, x\in [0,1]$, with $x = p(i,w)$, we can write
\begin{align*}
    q(i,v)
\ge
    1-\exp\left(-\sum_{w \in \cN(v)} p(i,w)\right)
=
    1-\exp\left(- P(i,v)\right)~.
\end{align*}
In turn, because $P(i,v) < 1$ in terms (II), we can use the inequality $1-e^{-x} \geq (1-e^{-1})\,x$, holding when $x \in [0,1]$, with $x = P(i,v)$, thereby concluding that
\[
 q(i,v) \ge (1-e^{-1})P(i,v)
\]

Thus
\begin{align*}
    \mathrm{(II)}
\le
    \sum_{v \in \cV\,:\, P(i,v) < 1} \frac{p(i,v)}{(1-e^{-1})P(i,v)}
\le
    \frac{1}{1-e^{-1}}\,\sum_{v \in \cV} \frac{p(i,v)}{P(i,v)}
\le
    \frac{\alpha_1}{1-e^{-1}}~,
\end{align*}
where in the last step we used Lemma~\ref{lm:bound-alpha}. 
\end{proof}

\end{document}